\documentclass[11pt]{article}
\usepackage[margin=1in]{geometry}
\usepackage{amsmath,amssymb,amsthm}
\usepackage{graphicx}
\usepackage{booktabs}
\usepackage{xcolor}
\usepackage[colorlinks=true,linkcolor=blue!50!black,citecolor=blue!50!black,urlcolor=blue!50!black]{hyperref}
\usepackage[numbers]{natbib}

\newtheorem{theorem}{Theorem}
\newtheorem{lemma}{Lemma}
\newtheorem{corollary}{Corollary}
\theoremstyle{remark}
\newtheorem{remark}{Remark}
\newtheorem{observation}{Observation}

\title{How fine a change can moments see?\\
A scale law for detecting distribution shift,\\
with a kernel calibration rule}
\author{Adel Kaleche\\\small Independent researcher, France \quad \url{https://kaleche.dev} \quad \texttt{kaleche@gmail.com}}
\date{}

\begin{document}
\maketitle

\begin{abstract}
Detecting that a stream of high-dimensional embeddings has changed is usually framed as a choice
of statistic. We give a scale law that constrains any moment-based choice, and we test what it
implies against topological alternatives. The law: certifying a feature of spatial scale
$\varepsilon$ carrying mass fraction $f$ requires polynomial tests of degree
$N^*\ge\log(1/f)/(2\varepsilon)$ --- proved via the Chebyshev extremal problem --- while a
Gauss-quadrature construction gives $N^*\ge4b-1$ for a $b$-scale topology, so the cost is governed
by feature \emph{fineness}, not by feature count. The law is one-sided: it never certifies that a
given order suffices, and we exhibit a counterexample (an annulus whose mean, covariance and
\emph{all} fourth-order moments equal those of a filled disk, yet $H_1\neq0$).

Its practical content is a calibration rule. The upper-bound construction attains its rate with
Gaussian test functions --- the RKHS witness of an RBF kernel --- so the law predicts \emph{which}
bandwidth an MMD test should use, namely the feature scale. On controlled deformations of real
embedding streams we measure $\sigma^\star/\varepsilon$ with median $1.12$ (interquartile range $1.01$--$1.52$,
$n=26$) over three settings and three scales under one protocol, and a bandwidth predicted from a data-driven scale estimate reaches
AUC $\ge0.95$.
Against an adversary optimised against the defender's own statistics $\{\mu,\Sigma,k\text{-NN},b\}$,
every one of those statistics is evaded and only a bandwidth-matched kernel test still detects.

For persistent homology the verdict is mixed and depends strongly on choices that are usually left
implicit. The summary matters more than the filtration: total persistence attains recall $0.75$ at
$\mathrm{FPR}=1\%$ on a covariance-preserving attack where the first persistence landscape attains
$0.00$; DTM weighting does not help. Once the summary is chosen sensibly, several findings we would otherwise have reported ---
degradation with ambient dimension, sensitivity to the reduction dimension, statistically
significant domination by kurtosis, failure to alarm in the streaming regime --- turn out to be
properties of the landscape rather than of persistence. What survives is a cost gap, not a power
gap: where persistence works it costs $116\times$ kurtosis, which works at least as well; where
kurtosis fails against an adaptive adversary, persistence fails with it.
We therefore do not conclude that topological summaries are useless, but that on this task they are
dominated by a kernel test whose bandwidth the law tells you how to set.
\end{abstract}

\section{Introduction}
Deciding that a stream of high-dimensional embeddings has changed is usually framed as a choice of
statistic: covariance drift, a density estimate, a two-sample test, a topological summary. We argue
that the choice is constrained by a quantity that is rarely made explicit --- the spatial
\emph{scale} of the change one hopes to catch --- and we give a law that says how the cost of
detection depends on it. The law bounds below the polynomial degree, hence the moment order, needed
to certify a feature of scale $\varepsilon$ carrying mass fraction $f$; its constructive side
identifies Gaussian test functions as the probes that attain the rate, which turns it into a
calibration rule for the bandwidth of a kernel two-sample test.

Embeddings make a natural testbed, and they come with a standing hypothesis worth testing against
the law. They concentrate near a low-dimensional submanifold of $\mathbb{R}^D$ (the \emph{manifold
hypothesis}), so attacks and drift are expected to manifest as \emph{topological} deformations ---
components merging, loops filling --- that a topological monitor could catch and purely local
measures would miss; persistent homology is the natural candidate, and a growing body of work
applies it to out-of-distribution and adversarial detection
\citep{carlsson2009topology,naitzat2020topology}. We subject that hypothesis to a deliberately
adversarial evaluation. The answer is not a clean negative: it depends on choices --- above all the
persistence summary --- that are usually left implicit, and what survives is a cost gap explained by
the scale law rather than a failure of topology.
Our contributions:
\begin{enumerate}
\item \textbf{A scale law} (\S\ref{sec:theory}): $N^*\ge\log(1/f)/(2\varepsilon)$ for polynomial
tests, proved; $N^*\ge4b-1$ for a $b$-scale topology, proved by Gauss quadrature; an upper bound
$C_d/(\varepsilon f)$ under an explicit margin hypothesis. Cost is set by fineness, not count.
\item \textbf{An explicit statement of what the law cannot do} (\S\ref{sec:false}): it is a lower
bound and certifies no sufficiency. The annulus counterexample matches all moments through order
four while carrying $H_1$, so no ``fourth-order information governs canonical topological change''
statement is available.
\item \textbf{A calibration rule with a measured constant} (\S\ref{sec:sigma}):
$\sigma^\star\approx\varepsilon$ for RBF-MMD (median $1.12$, IQR $[1.01,1.52]$, $n=26$ under one
protocol over three settings and three scales), together with the finding that the argmax estimator
is high-variance so that single-run ratios must not be over-read; a data-driven scale estimate
closes the loop (with an explicitly stated compensation of biases), and we find no resolution floor
down to $0.28$ of the inter-point spacing.
\item \textbf{A hardened empirical comparison} (\S\ref{sec:emp}) --- three embedders, two corpora,
an adversary optimised against the defender's statistics, operating points at fixed FPR, streaming
delay at fixed ARL, cost table, four persistence summaries, two filtrations, a reduction-dimension
sweep --- reported with the settings that favour persistence as well as those that do not.
\end{enumerate}

\paragraph{Reproducibility.} All experiments are short, self-contained Python scripts (numpy,
ripser, scikit-learn, sentence-transformers); every number quoted below is produced by one of them.
Code and the exact scripts behind each table are archived at \href{https://doi.org/10.5281/zenodo.21649324}{doi:10.5281/zenodo.21649324} (live repository: \url{https://github.com/Adelagric/moment-scale-law}).

\paragraph{On repeated cells.} Tables produced by different experiments re-draw their windows
independently, so a nominally identical cell can differ between tables by sampling variation. The
first-landscape AUC on the covariance-preserving attack at ambient $1024$, reduction dimension $9$,
reads $0.64$ (Table~\ref{tab:dim}), $0.70$ (Table~\ref{tab:robust}) and $0.75$
(\S\ref{sec:summary}) across three such runs; total persistence reads $1.00$, $0.97$ and $0.99$.
The bootstrap interval reported in Table~\ref{tab:robust} --- $[0.53,0.85]$ for the landscape ---
covers this spread, and we quote each table's own run rather than silently propagating one number. Where a measurement contradicted an earlier claim of ours, we report the
correction explicitly rather than the original claim.
We view the result as a methodological guardrail: before proposing a topological feature, test
whether it is redundant with a low-order moment fixed by the geometry of the target effect.

\section{Setup and related work}
\textbf{Pipeline.} Unit-normalised embeddings $x_t\in\mathbb{R}^D$ are cut into sliding windows
$\mathbb{X}_t$ of size $W$ (all experiments use $W=200$ unless stated). We reduce each window to its intrinsic dimension by PCA before any
topological computation --- a step supported by the null-model analysis of \S\ref{sec:reduce} (raw high-$D$
Vietoris--Rips is dominated by distance concentration), though not by the detection task itself
--- see \S\ref{sec:op}. On the reduced cloud we compute
the Vietoris--Rips diagram, its first persistence landscape $\lambda$ \citep{bubenik2015landscapes},
and a scalar $\Delta_{\mathrm{top}}=\lVert\lambda-\bar\lambda_{\mathrm{ref}}\rVert_{L^1}$.

\textbf{Two-sample testing and drift detection.} Kernel two-sample tests \citep{gretton2012kernel},
energy distance \citep{szekely2013energy} and their learned-kernel successors
\citep{liu2020deepkernel} are the standard tools for detecting distribution shift; most directly,
\citet{rabanser2019failing} benchmark exactly our setting --- shift detection on learned
representations, with dimensionality reduction upstream of a two-sample test --- and find
two-sample testing on reduced representations to be the strongest family. Our contribution
relative to that literature is not a new detector but a \emph{scale law} for the one that works,
plus the comparison against a topological alternative.

\textbf{Classical baselines.} covariance drift $\lVert\Sigma-\Sigma_{\mathrm{ref}}\rVert_F$, written \textsc{cov-drift} below;
Mahalanobis shape; $k$-NN density (related to LID; \citep{ma2018lid}); Mardia kurtosis
$b=E[(x-\mu)^\top\Sigma^{-1}(x-\mu)]^2$ \citep{mardia1970kurtosis}.

\textbf{Related work.} VR/\v{C}ech stability \citep{cohensteiner2007stability,chazal2009gromov};
DTM and robust/sparse persistence
\citep{chazal2011geometric,anai2020dtm,buchet2016efficient,sheehy2013linear}; landscapes and
stable vectorisations \citep{bubenik2015landscapes,adams2017images}; witness complexes
\citep{desilva2004witness}. TDA has been applied to detection and network analysis:
\citet{gebhart2017adversary} detect adversarial inputs from the persistent homology of activation
graphs (98\% on MNIST), \citet{naitzat2020topology} characterise how depth simplifies the topology
of decision regions, and \citet{ballester2023survey} survey the area. These works, however,
(i) operate on \emph{activation graphs} rather than embedding point clouds, (ii) evaluate against
non-adaptive adversaries, and (iii) compare to accuracy or cosine/softmax baselines rather than a
\emph{higher-moment} or \emph{density} baseline. Our evaluation closes exactly that gap --- an adversary optimised against the defender's own
statistics, and a battery spanning covariance, Mahalanobis, $k$-NN density, kurtosis, a radial
Kolmogorov--Smirnov test, energy distance and MMD-RBF --- and supplies the
closed-form reason why the comparison turns out as it does.

\section{A moment-order law for topological detection}\label{sec:theory}
Let $X=R\cdot U$ in $\mathbb{R}^d$, $U$ uniform on $S^{d-1}$, $R\ge 0$ independent of $U$. The
radial law encodes the topology: $R\equiv 1$ is a shell (carries $H_{d-1}$); $R$ with density
$\propto r^{d-1}$ is a filled ball.

\begin{lemma}[blindness of second order]
$E[U]=0$, $E[UU^\top]=\tfrac1d I$, hence $E[X]=0$ and $\operatorname{Cov}(X)=\tfrac{E[R^2]}{d}I$.
Mean and covariance depend only on $E[R^2]$.
\end{lemma}

With $\Sigma^{-1}=\tfrac{d}{E[R^2]}I$ and $c=E[R^2]$,
\[
b_{\mathrm{Mardia}}=E[(X^\top\Sigma^{-1}X)^2]=\frac{d^2}{c^2}E[R^4]
=d^2\Big(1+\frac{\operatorname{Var}(R^2)}{c^2}\Big).
\]

\begin{remark}[minimality of the shell; Jensen]
At fixed $E[R^2]$, among spherically symmetric clouds the degenerate shell $R\equiv\sqrt c$
minimises the fourth moment, with $\operatorname{Var}(R^2)=0$, $b=d^2$. Along the one-parameter
family interpolating shell $\to$ filled ball, $b$ increases strictly with
$\operatorname{Var}(R^2)$. Hence mean and covariance are structurally blind to the radial shape
(hence to $H_{d-1}$), while the fourth moment is monotone along this family --- with a mechanism
independent of the homological dimension $d-1$.
\end{remark}

Numerically: shell $b=d^2$ predicts $4,9$ ($d=2,3$), measured $4.04,8.99$; for the unit ball
$E[R^2]=\tfrac{d}{d+2}$, $E[R^4]=\tfrac{d}{d+4}$ give $b=\tfrac{16}{3}=5.33$ and
$\tfrac{75}{7}=10.71$, measured $5.28,10.75$.

\paragraph{The converse is false: $b$ does not determine the topology.}
It is tempting to read the above as ``persistence and kurtosis are two readouts of one
quantity''. That is wrong, and the counterexample is elementary. Any annulus of nonzero thickness
carries $H_{d-1}$, and $b$ depends only on the scalar
$\operatorname{Var}(R^2)/c^2$, so the map topology $\to b$ is not injective. Concretely in $d=2$, take the radial law $R\sim\mathrm{Unif}[a,1]$ (planar density
$\propto1/r$ on the annulus, not uniform on the region):
\[
\frac{E[R^4]}{E[R^2]^2}=\frac95\cdot\frac{1+a+a^2+a^3+a^4}{(1+a+a^2)^2},
\]
which equals the filled-disk value $4/3$ at $a^\ast=0.2954$. That annulus has \emph{exactly} the
mean, covariance and Mardia kurtosis of the filled disk, yet $H_1\neq0$ (on $n=1200$ points per cloud: maximal
$H_1$ persistence $0.455$ vs.\ $0.118$; $b=5.26$ vs.\ $5.28$; $\lvert\Delta\Sigma\rvert_F=0.04$,
all at the sampling floor). It is a coarse, single-scale hole --- precisely the regime where our
empirical study finds kurtosis sufficient. The correct statement is therefore \emph{monotonicity
along the shell--ball family}, not equivalence: our empirical redundancy is a statement about the
attacks we generate (which move along such families), not a theorem that kurtosis sees every hole.

\begin{corollary}[elliptical extension]
Mardia kurtosis is affine-invariant ($b(MX)=b(X)$), and an invertible affine map is a
homeomorphism, so it preserves \emph{homotopy type} --- though \emph{not} the Vietoris--Rips
filtration, which depends on the metric. The moment side of the statement therefore transfers
verbatim to every elliptically symmetric family, with $R$ the Mahalanobis radius, and the
topological side transfers provided persistence is computed after whitening (our pipeline does not whiten): a shell has $b=d^2$ regardless of anisotropy (measured
$4.04,9.00$ under
an affine map of condition number $3.7,6.3$). This addresses the ``real clusters are not
spherical'' objection.
\end{corollary}

\begin{figure}[t]\centering
\includegraphics[width=0.72\textwidth]{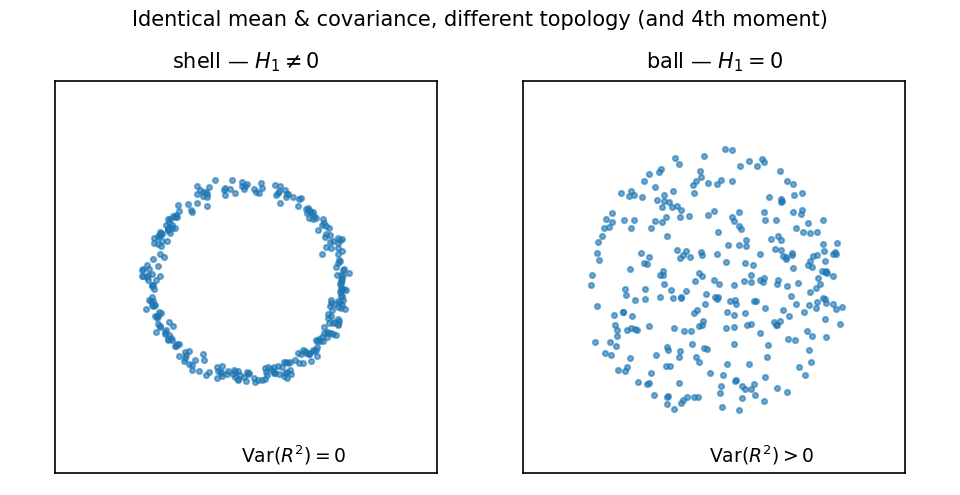}
\caption{Shell vs.\ ball with identical mean and covariance: the topology (and the fourth
moment, via $\operatorname{Var}(R^2)$) differ, while first- and second-order statistics do not.}
\label{fig:shell}
\end{figure}

\textbf{Trade-off curve and the geometry law.} Starting from a broken loop and matching a growing
statistic set, $H_1$ stays hidden under $\{\mu,\Sigma,k\text{-NN density}\}$
($H_1/h^*\approx0.07$) and jumps at kurtosis ($0.77$; Fig.~\ref{fig:tradeoff}). Across the
sphere/ball family, $H_1$ and $H_2$ both threshold at the fourth moment; $H_0$ thresholds at
covariance when separation raises variance, but at kurtosis once clusters are variance-matched.
Empirically, then, the order at which the topology is forced tracks the lowest moment the
geometry of the change perturbs rather than the homological dimension. We state this as an
observation about these families, not as a law: the bound is one-sided and cannot certify
sufficiency.

\begin{observation}[the shell--ball family, across homological dimension]
Along the shell$\to$ball family at fixed $E[R^2]$, the fourth moment is monotone, and the same
holds for $H_0$: for a symmetric two-component mixture in $\mathbb{R}^d$ with masses at $\pm a$
along one axis and within-component variance $w^2$, $E[X_1^2]=a^2+w^2$ and
$E[X_1^4]=a^4+6a^2w^2+3w^4$, so with $t=a^2$, $s=w^2$ the standardised
fourth moment along the separation axis is $\kappa_1=3-2t^2/(t+s)^2<3$ for all $a>0$. For
independent standardised coordinates $b_{\mathrm{Mardia}}=\sum_i\kappa_i+d(d-1)$, hence
\[
b \;=\; d(d+2)\;-\;\frac{2t^2}{(t+s)^2}\;<\;d(d+2),
\]
so the bimodal cloud ($H_0=2$) is strictly platykurtic --- a closed form, not a measurement
(numerically at $d=3$: $b=13.59$). The comparison cloud we use is Gaussian, for which
$b=d(d+2)=15$ exactly (measured $15.02$); a non-Gaussian unimodal control would make the
contrast sharper. The mechanism is thus the same for $H_0$, $H_1$ and $H_2$, and independent of the
homological dimension.

\emph{This is monotonicity along a family, not a characterisation.} The next paragraph shows the
converse fails outright, so no statement of the form ``fourth-order information governs canonical
topological change'' is available --- in spherical symmetry $E[R^4]$ \emph{is} the entire
fourth-order information, and our counterexample matches it exactly.
\end{observation}

\begin{theorem}[complexity--moment lower bound; proved for all $b$, all $d\ge2$]
For every $b\ge1$ there exist probability measures $\mu_b,\nu$ on $\mathbb{R}^d$ with
$\int p\,d\mu_b=\int p\,d\nu$ for all polynomials $p$ of degree $\le 4b-2$, yet
$\beta_{d-1}(\operatorname{supp}\mu_b)=b$ and $\beta_{d-1}(\operatorname{supp}\nu)=0$. Hence
$N^*(b)\ge 4b-1$: the moment order needed to distinguish a topology of complexity $b$ from a
trivial one grows at least linearly in $b$.
\end{theorem}

\begin{proof}[Proof sketch]
Let $\nu$ be a rotation-invariant density on the unit ball (radial law $\rho$ on $u=\lVert
x\rVert^2$) and $\{(u_i,w_i)\}_{i=1}^b$ its $b$-point Gauss quadrature; set
$\mu_b=\sum_i w_i\,\sigma_{\sqrt{u_i}}$ ($b$ concentric $(d{-}1)$-spheres, $\beta_{d-1}=b$). For
$\deg p\le 4b-2$, the rotational average $\bar p$ is a polynomial in $u$ of degree $\le 2b-1$; by
rotation invariance $\int p\,d\mu=\int\bar p\,d\mu$, and Gauss exactness (degree $\le 2b-1$) gives
$\sum_i w_i\bar p(u_i)=\int\bar p\,d\rho$, i.e.\ $\int p\,d\mu_b=\int p\,d\nu$.
\end{proof}

In the radial case the bound is tight \emph{within the class of $b$-atom radial laws} (Hamburger
uniqueness fixes such a law by its first $2b$ moments), so $N^*=\Theta(b)$ there; the same holds for elliptical, axis-separable, and
group-symmetric topology by reduction. Persistence reads all $b$ features from one barcode --- a
bounded-computation surrogate for an unbounded-order moment (Fig.~\ref{fig:frontier}; verified
$m=1,2,3$, matched orders $2,6,10$). No fixed-order moment is redundant with persistence for rich
topology.

A matching \emph{upper} bound $N^*\le C_d\,b$ is, however, \textbf{false} as stated: a feature of
radius $\varepsilon$ carrying mass fraction $f$ perturbs moments by only $O(f\varepsilon^k)$, so
sufficiently fine topology escapes \emph{every} fixed moment order. Numerically (ring vs.\ disk of
radius $\varepsilon$, $f=0.12$, normalised moments up to order $10$), the relative moment
discrepancy collapses $5.6\!\cdot\!10^{-2}\to 2.2\!\cdot\!10^{-4}$ as $\varepsilon:0.5\to0.03$
while the $\varepsilon$-scale homology persists. Any upper bound must therefore be \emph{scale-relative}. We determine the scale law empirically:
in an orthonormal (tensor-Legendre) basis, the smallest order at which the coefficients of an
$\varepsilon$-ring and an $\varepsilon$-disk diverge obeys
\[
N^*(\varepsilon)\;\approx\;0.74\,\varepsilon^{-0.92},\qquad R_{\log\log}=0.983,
\]
consistent with the theoretical exponent $1$: order-$N$ moments resolve scale $\sim1/N$
(Christoffel--Darboux), so seeing an $\varepsilon$-feature costs $N\sim1/\varepsilon$.

Crucially, the cost is \emph{not} a product with the feature count: at fixed $\varepsilon$,
increasing $b$ from $1$ to $16$ raises $N^*$ only from $4$ to $8$ (empirical exponent $\approx
0.25$ in $d=2$), because a single degree-$N$ polynomial resolves scale $1/N$ \emph{uniformly},
hence all features at once. The $b$ in the lower bound is not a counting cost: packing $b$ shells
into the unit ball forces $\varepsilon\sim1/b$. This yields a matching upper bound with no $b$ at
all:

\begin{theorem}[scale-relative upper bound, under a margin hypothesis]\label{thm:upper}
Let $\mu,\nu$ be probability measures on $[-1,1]^d$ agreeing on all moments of degree $\le N$.
Assume the \emph{$f$-margin condition}: every ball of radius $\varepsilon$ has $\mu$- and
$\nu$-mass either $\ge 2f$ (``occupied'') or $\le f/2$ (``empty''). Then $\mu$ and $\nu$ induce the
same $\varepsilon$-occupancy pattern, hence the same nerve and the same $\check{\mathrm C}$ech
Betti numbers at radius $\varepsilon$, as soon as $N\ge C_d/(\varepsilon f)$.
\end{theorem}

The margin hypothesis is not cosmetic: occupancy is a thresholded quantity, so without a gap an
arbitrarily small mass discrepancy can flip one ball, hence the nerve, hence the Betti numbers ---
the statement is refutable already in $d=1$. ``Betti numbers at scale $\varepsilon$'' means those
of the $\check{\mathrm C}$ech complex at radius $\varepsilon$.

\begin{proof}[Proof sketch]
Fix $x$ and take $\varphi_x$ smooth with $\varphi_x\equiv1$ on $B(x,\varepsilon/2)$,
$\mathrm{supp}\,\varphi_x\subseteq B(x,\varepsilon)$, $0\le\varphi_x\le1$; then
$\omega(\varphi_x,\delta)\lesssim\delta/\varepsilon$. This pairing matters: a test function merely
supported in $B(x,\varepsilon)$ does not lower-bound $\int\varphi_x\,d\mu$ for an occupied ball,
since the mass could sit near the boundary where $\varphi_x$ vanishes. With $\varphi_x\equiv1$ on
the half-radius ball, $\mu(B(x,\varepsilon/2))\le\int\varphi_x\,d\mu\le\mu(B(x,\varepsilon))$.
By Jackson's theorem there is $p$, $\deg p\le N$, with
$\lVert\varphi_x-p\rVert_\infty\le C_d\,\omega(\varphi_x,1/N)\le C_d/(N\varepsilon)$; since
$\deg p\le N$ we have $\int p\,d\mu=\int p\,d\nu$, whence
$\lvert\int\varphi_x\,d\mu-\int\varphi_x\,d\nu\rvert\le2C_d/(N\varepsilon)$, which is $<f/2$ once
$N\ge C_d'/(\varepsilon f)$. Under the margin hypothesis stated on both radii, the sign of
$\int\varphi_x\,d\mu-\tfrac{3}{2}f$ therefore agrees for $\mu$ and $\nu$ at every $x$, so both
measures induce the same occupancy labelling of any fixed $\varepsilon$-net $\mathcal N$. The
conclusion is about the nerve of $\{B(x,\varepsilon)\}_{x\in\mathcal N}$ --- a fixed cover ---
and hence about its Betti numbers; relating that nerve to the \v{C}ech complex of the sample
requires a separate interleaving argument which we do not carry out, so we state the result for
the net's nerve only.
\end{proof}

Two honest caveats on this theorem. The margin hypothesis is \emph{assumed}, not derived from
``a feature of scale $\varepsilon$ and mass $f$''; it encodes a good part of the conclusion. And
the statement concerns the nerve of a fixed $\varepsilon$-net, not the \v{C}ech complex of the
sample.

The $1/f$ factor comes from first-order Jackson (a Lipschitz bump) and can be removed entirely by
choosing better test functions. Compactly supported Gevrey-$s$ profiles give
$\lVert\varphi_\varepsilon-p_N\rVert\lesssim\exp(-s(N\varepsilon/CA)^{1/s})$, hence
$N^*\lesssim\varepsilon^{-1}(\log 1/f)^{s}$; \emph{Gaussian} test functions do better, trading
compact support for analyticity, and numerically optimising their width yields
$N^*\propto(\log 1/f)^{1.04}$ versus $(\log 1/f)^{1.81}$ for the Gevrey family
($R\ge0.997$; at $f=10^{-6}$, $N^*=255$ vs.\ $1799$). This is optimal: by the Chebyshev extremal
problem, a degree-$N$ polynomial localised at scale $\varepsilon$ has contrast at most
$T_N\!\big(\tfrac{1+\varepsilon^2}{1-\varepsilon^2}\big)\approx e^{2N\varepsilon}$ (measured ratio
$\log(\text{contrast})/(N\varepsilon)\to1.99$), so distinguishing mass $f$ from $0$ by any \emph{polynomial} test of degree $\le N$ requires
\[
\boxed{\;N^*(\varepsilon,f)\;\ge\;\frac{\log(1/f)}{2\varepsilon}\;}
\]
This lower bound is proved, and its scope must be stated precisely: it constrains \emph{polynomial} (moment) tests of degree $\le N$, not all statistical tests. It is not an information-theoretic impossibility --- MMD is not a polynomial test, and \S\ref{sec:emp} shows it succeeds where moment tests fail. On the upper side we have a \emph{proved} bound $C_d/(\varepsilon f)$
(Theorem~\ref{thm:upper}) and \emph{numerical evidence} --- a fitted exponent, not a proof --- that Gaussian test
functions attain $O(\log(1/f)/\varepsilon)$, which would make the law tight. Closing this requires
bounding the Chebyshev coefficients of an $\varepsilon$-wide Gaussian via a Bernstein ellipse
argument; we have not carried it out, and we do not claim tightness.
The feature count enters only through packing: $b$ features in a bounded domain force
$\varepsilon,f\lesssim1/b$, giving $N^*\gtrsim b\log b$ --- sharper than the $4b-1$ obtained from
the explicit Gauss construction, and consistent with it.

The mechanism behind both bounds --- that order-$N$ moments resolve scale $1/N$, and that
smoothness converts this into super-algebraic accuracy --- is confirmed numerically over $14$
decades (Fig.~\ref{fig:asym}); we record it for the Gevrey family, whose exponent the Gaussian
construction above then improves on. Computing
Fourier coefficients of an $\varepsilon$-bump by FFT, a standard (Gevrey-$2$) mollifier satisfies
$\log|\hat c_k|\propto -(k\varepsilon)^{1/2}$ with correlation $-0.993$ (versus $-0.960$ for an
algebraic fit, whose fitted exponent $-7.3$ is itself a symptom of non-algebraic decay), while a
Lipschitz bump stays algebraic with the predicted exponent $-1.94\approx-2$. The measured slope in
$(k\varepsilon)^{1/2}$ is $-1.955$ and $-1.947$ at $\varepsilon=0.15,0.30$ --- essentially
identical, confirming that the decay depends on the product $k\varepsilon$ with a universal
constant $c\approx1.95$. Hence $\mathrm{err}(N)\approx\exp(-c(N\varepsilon)^{1/2})$ for Gevrey-$2$ test
functions --- the $s=2$ instance of the intermediate bound, with its constant measured --- which
the Gaussian family improves to the law boxed above.

\begin{figure}[t]\centering
\includegraphics[width=0.92\textwidth]{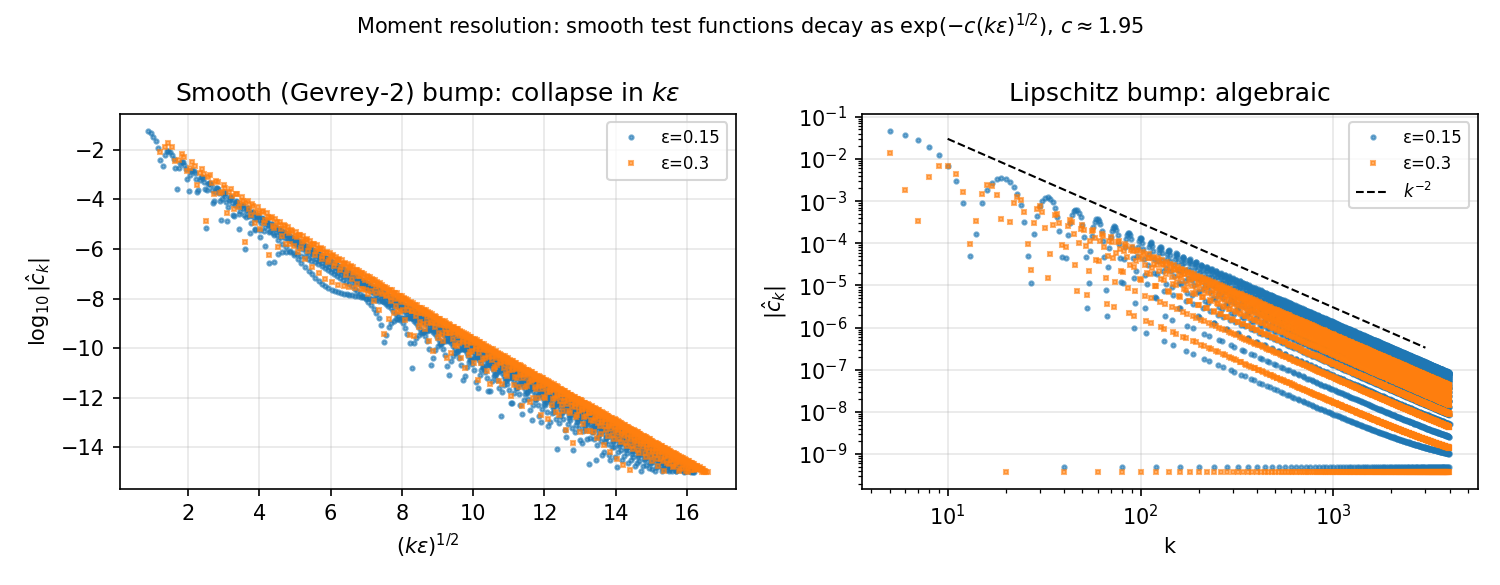}
\caption{Left: Fourier coefficients of a smooth (Gevrey-2) $\varepsilon$-bump collapse onto a
single line in $(k\varepsilon)^{1/2}$ across $\varepsilon$, i.e.\ $\exp(-c(k\varepsilon)^{1/2})$
with $c\approx1.95$. Right: a Lipschitz bump decays algebraically as $k^{-2}$. Order-$N$ moments
resolve scale $1/N$; smooth test functions convert that into super-algebraic accuracy, which is
what reduces the mass dependence from $1/f$ to $(\log 1/f)^2$.}
\label{fig:asym}
\end{figure}

The law is thus qualitatively clear: \textbf{$N^*$ is governed by feature
\emph{fineness} ($\varepsilon$, $f$), essentially not by feature \emph{count}}. This strengthens
the case for persistence: a barcode reads all scales at once, including fine features whose moment
cost diverges as $\varepsilon\to0$.

For $m=1$ the bound reads $N^*\gtrsim4$; this is consistent with (but does not imply) the
kurtosis redundancy measured in \S\ref{sec:emp}: the empirical
detection regime is exactly the simple-topology case $m=1$, where a low-order moment shadows
persistence. The genuine value of persistence is the large-$m$ regime --- rich multi-scale topology,
where the equivalent moment order is prohibitive. This delimits, rather than dismisses, topological
features: use them when the topology is complex, not for canonical single-scale changes.
The law is not an artifact of spherical symmetry: for $b$ loops placed along an axis at the
Gauss nodes of a contractible tube's marginal, the $x$-moments agree up to order $2b-1$ while
$H_1=b$ vs.\ $0$ (verified $b\le 4$), so $N^*\propto b$ holds for any \emph{axis-separable}
topology. Group-symmetric topology reduces to a lower-dimensional quotient where the same law applies
(a torus of revolution reduces to its $(\rho,z)$ profile, its tube cycle becoming a
circle-vs-disk / fourth-profile-moment feature), so the genuinely open case narrows to
\emph{asymmetric} entanglement. An adversarial search for a persistence-visible counterexample
turned up none: linked cycles share Betti numbers and are invisible to ordinary persistence;
measure-zero perturbations vanish as $N\to\infty$; and geometric complexity of a single feature
(a $k$-armed star loop, $\beta_1=1$, strongly varying radius) does not inflate $N^*$ --- at
matched covariance the fourth moment still detects its hole at AUC $1.0$ for $k$ up to $8$.
Consistently with the law above, what these probes vary --- arrangement, entanglement,
shape --- leaves $N^*$ unchanged, since none of them alters feature scale or mass. We conjecture
the law extends to arbitrary curved manifolds via a nerve (local-to-global) argument; a formal
reduction for entangled topology remains open.

\begin{figure}[t]\centering
\includegraphics[width=0.6\textwidth]{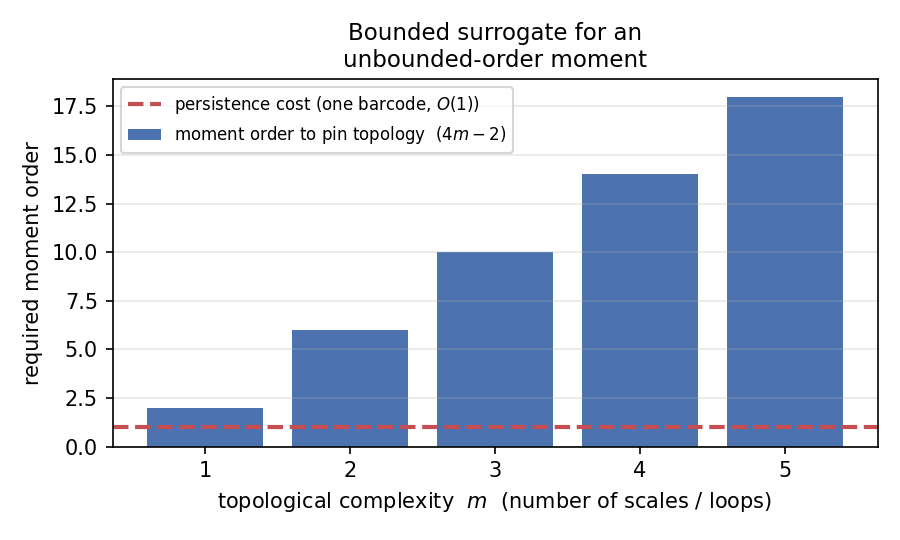}
\caption{Moment order $4m-2$ needed to pin an $m$-scale topology grows linearly, while persistence
reads all scales from a single barcode. Persistence is a bounded surrogate for an unbounded-order
moment.}
\label{fig:frontier}
\end{figure}

\begin{figure}[t]\centering
\includegraphics[width=0.62\textwidth]{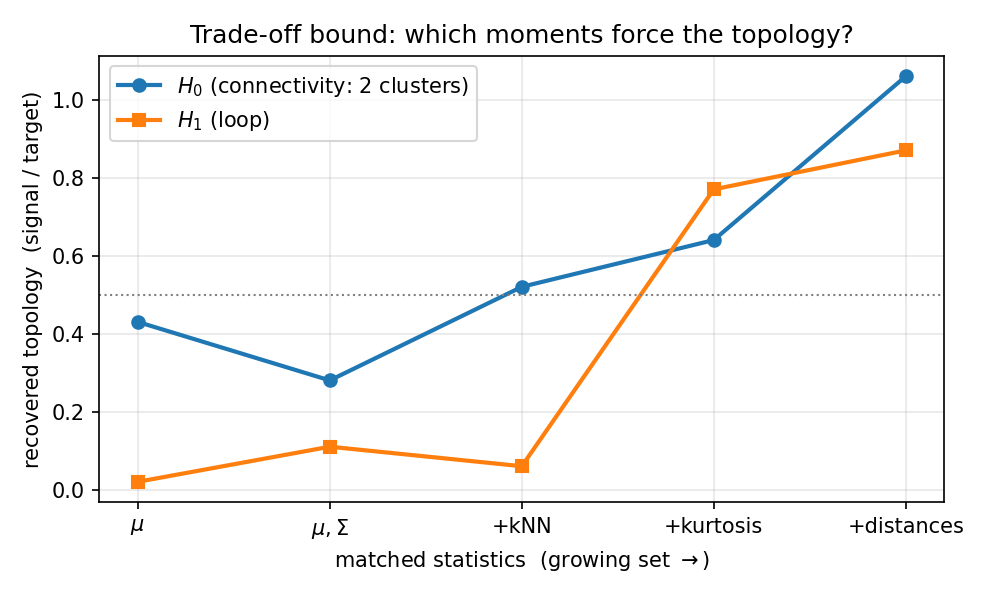}
\caption{Hidden topology (signal / target) vs.\ matched statistics. $H_1$ is hideable under
$\mu,\Sigma$ and local density, and is forced only at the fourth moment.}
\label{fig:tradeoff}
\end{figure}

\section{The strong conjecture is false (honest scope)}\label{sec:false}
One might conjecture that any homology change perturbs some low-order moment. This is false. By
gradient optimisation we drive a broken loop to match a circle's mean, covariance and full $k$-NN
distance distribution to numerical zero ($\lvert\Delta\Sigma\rvert=0$, $k$-NN RMSE
$2\times10^{-4}$) while keeping $H_1$ broken ($H_1/h^*=0.13$). PH thus carries information
orthogonal to covariance and local density --- there is no impossibility barrier. However, such
configurations are non-generic and require explicit optimisation, and leave a residual higher
moment (kurtosis AUC $=1.00$ even at $\lvert\Delta b\rvert=0.008$). The defensible claim is about
\emph{generic} changes.

\section{What the law predicts}\label{sec:predict}
Two consequences follow, and the rest of the paper tests them.

\textbf{(H1) Coarse, single-scale topology \emph{may} be redundant --- a heuristic conjecture,
not a consequence of the law.} Our law is a \emph{lower} bound: $N^*\ge\log(1/f)/(2\varepsilon)$
constrains from below and can never yield a statement of sufficiency such as ``order $4$ suffices''.
Worse, our own counterexample refutes sufficiency in exactly this regime: for the annulus of
\S\ref{sec:false} ($\varepsilon\approx0.3$, $f=\Theta(0.1)$) the bound gives $N^*\gtrsim3.8$, and
order $4$ \emph{fails} --- all moments through order four coincide with the filled disk while
$H_1\neq0$. So $N^*>4$ for a coarse, single-scale hole. What remains is the conjecture that for the
deformations arising in practice --- which move monotonically along the shell--ball family --- a
fourth moment does suffice. We test that conjecture rather than deduce it.

\textbf{(P2) The optimal probe is a Gaussian, i.e.\ a kernel two-sample test.} The upper-bound
construction attains its rate with Gaussian test functions of width matched to the feature scale.
A Gaussian test function integrated against $\mu-\nu$ is exactly the RKHS witness of an RBF
kernel, so the law predicts that MMD-RBF \emph{at bandwidth matched to the feature scale} is
near-optimal --- and, by the same token, that MMD at a mismatched bandwidth is weak. Notably, the
law does \emph{not} predict that persistence is optimal: a barcode is scale-free, which is an
advantage only when the relevant scale is unknown.

Section~\ref{sec:emp} tests H1 and P2 on real embedding streams under an adaptive adversary.

\paragraph{Order of discovery.} In fairness to the reader: H1 was pre-registered as the study's
hypothesis, whereas P2 was formulated \emph{after} the initial negative result and before the
kernel experiments reported in \S\ref{sec:kernel}. The narrative order of this paper is not its
chronological order.

\section{Testing the prediction on embedding streams}\label{sec:emp}
\textbf{Pre-registration.} Win/kill criteria were fixed before running: PH ``wins'' only if it
beats the best baseline at a production false-positive regime and not by inflating false
positives on legitimate topic shifts; the detection axis is falsified if PH beats neither
Mahalanobis nor covariance drift on any attack class.

\subsection{Higher homology exists only after reduction}\label{sec:reduce}
On real embeddings (MiniLM, 20\,Newsgroups), raw $384$-d $H_1$ is indistinguishable from a
covariance-matched Gaussian null (median bootstrap $p\approx0.95$); the intrinsic dimension
(TwoNN; \citep{facco2017twonn}) is $\approx 8$ while $160$ PCA axes are needed for $90\%$
variance --- a \emph{curved} low-dimensional manifold. After PCA to the intrinsic dimension,
multi-topic windows show a population-level $H_1$ signal above the null (median $p=0.03$;
sign-test $p=0.001$; confidence bands as in \citep{fasy2014confidence}). Reduction-before-TDA
is thus a first-order requirement.

\subsection{First pass: cheap statistics match or beat persistence}
On realistic attacks (off-manifold injection, diversity collapse) \textsc{cov-drift} reaches
AUC $1.00$ while $\Delta_{\mathrm{top}}$ is at or below chance. We then construct the attack most
favourable to PH: a diversity collapse followed by an affine correction restoring the pre-image's
\emph{exact} mean and covariance ($\lvert\Delta\Sigma\rvert=0$), so \textsc{cov-drift} is blind by construction.
Even so, the shape change is caught by Mahalanobis and by kurtosis (Table~\ref{tab:dim},
Fig.~\ref{fig:dim}).

\begin{table}[t]\centering
\caption{Covariance-preserving attack, bilateral AUC, $\lvert\Delta\Sigma\rvert=0$. The two
persistence columns share a diagram and differ only in the summary; the apparent degradation with
ambient dimension is a property of the landscape.}
\label{tab:dim}
\begin{tabular}{rcccccc}
\toprule
ambient & landscape & total pers. & \textsc{cov-drift} & Mahalanobis & $k$-NN & Kurtosis\\
\midrule
384  & 0.89 & \textbf{1.00} & 0.66 & 1.00 & 0.91 & \textbf{1.00}\\
768  & 0.77 & 0.93 & 0.66 & 0.98 & 0.98 & \textbf{1.00}\\
1024 & 0.64 & \textbf{1.00} & 0.54 & 1.00 & 0.98 & \textbf{1.00}\\
\bottomrule
\end{tabular}
\end{table}

\begin{figure}[t]\centering
\includegraphics[width=0.62\textwidth]{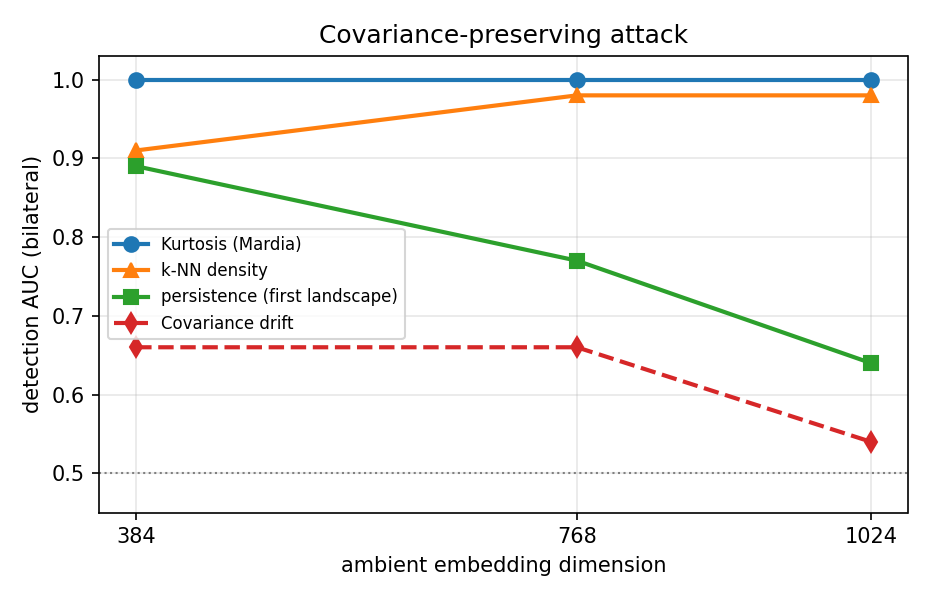}
\caption{Only pure covariance drift is blind; every other cheap detector catches the
covariance-preserving attack. The plotted persistence curve is the \emph{first landscape}, which
degrades with ambient dimension while kurtosis stays at $1.00$; total persistence, read off the
same diagrams, does not ($1.00,0.93,1.00$; Table~\ref{tab:dim}), so this degradation is a property
of the summary and we withdraw it as a statement about persistent homology
(\S\ref{sec:summary}). Intrinsic dimension is stable ($\approx 8/7/9$).}
\label{fig:dim}
\end{figure}

\subsection{The idealised topological niche is also covered}
On a synthetic circle-vs-disk pair with covariance matched by construction --- a purely
topological difference invisible to \textsc{cov-drift} (AUC $0.58$) --- $\Delta_{\mathrm{top}}$ separates
perfectly ($1.00$), but so does $k$-NN density ($1.00$): even the ideal blind spot of
second-order statistics is covered by a cheap density estimator. (Fig.~\ref{fig:shell} shows the
shell/ball geometry underlying the construction, not these detection results.)

\subsection{Kernel two-sample tests, and a genuinely adaptive adversary}\label{sec:kernel}
Our own frontier analysis (\S\ref{sec:theory}) identifies a Gaussian test function of width
$\varepsilon$ as the near-optimal probe for an $\varepsilon$-scale feature --- that is precisely
the RKHS witness of an RBF kernel. We therefore add MMD-RBF, energy distance and a
Kolmogorov--Smirnov test on the Mahalanobis radial law (which subsumes kurtosis) to the battery,
and we strengthen the adversary: rather than merely preserving $(\mu,\Sigma)$, it is
\emph{optimised} against the defender's statistics $\{\mu,\Sigma,k\text{-NN},b\}$ by gradient
descent, never touching persistence.

Two findings, both against our initial thesis. First, MMD's power is entirely a matter of
bandwidth, exactly as the frontier law predicts: on the covariance-preserving attack, MMD-RBF at
the median heuristic is near-blind (AUC $0.60$) and reaches $1.00$ once
$\sigma\le\text{median}/4$. Second, against the genuinely adaptive adversary (residuals
$\lvert\Delta\mu\rvert=0.0000$, $\lvert\Delta\Sigma\rvert_F=0.0002$, $\lvert\Delta b\rvert=0.000$),
every statistic it optimises against is evaded --- \textsc{cov-drift} $0.64$, $k$-NN $0.61$,
kurtosis $0.61$ --- \emph{and persistence does not fill the gap} ($\Delta_{\mathrm{top}}=0.60$,
chance). The detector that does is MMD-RBF at matched bandwidth ($1.00$; KS-radial $0.95$).

The evasion channel is therefore real, but it is closed by a kernel two-sample test, not by
topology. This is the sharpest form of our negative result, and our own theory predicts it: the
optimal test function is Gaussian, so the optimal practical detector is an RBF-kernel test at the
feature scale. Persistence's only structural advantage over it is that a barcode is scale-free,
whereas MMD must choose a bandwidth --- but multi-bandwidth MMD is standard and cheap.

\paragraph{An MMD-aware adversary.} The adversary above never had the winning detector in its
objective --- it optimised against $\{\mu,\Sigma,k\text{-NN},b\}$ and left both persistence and MMD
untouched --- so ``only a kernel test detects'' was, strictly, a statement about a detector the
attacker was not attacking. We therefore added a differentiable MMD term at the matched bandwidth to
the objective. The result is that the attack fails to evade it: MMD still separates at AUC $1.00$
while kurtosis ($0.53$), $k$-NN ($0.57$) and total persistence ($0.55$) are all evaded, and the
attempt breaks the attacker's own second-order cover (\textsc{cov-drift} rises to $1.00$). At this
optimisation budget the constrained problem --- alter the topology while matching mean, covariance,
$k$-NN law, kurtosis \emph{and} the kernel statistic --- appears infeasible. We state the obvious
caveat: this is one optimiser at one budget, not an impossibility, and a better-resourced
MMD-aware adversary is the natural next escalation.

\subsection{The law as a calibration rule: $\sigma^\star\approx\varepsilon$}\label{sec:sigma}
That bandwidth matters is well known \citep{gretton2012kernel}; the median heuristic is a poor
default under local alternatives. The non-trivial content of the law is \emph{which} bandwidth. We
test it with a controlled deformation: replace a mass fraction $f$ of a window by a localised blob
of \emph{known} radius $\varepsilon$, sweep $\sigma$, and record
$\sigma^\star=\arg\max\mathrm{AUC}$.

\paragraph{Identifiability first.} At large $f$ the AUC saturates at $1.00$ over a wide plateau and
the argmax is not defined; an earlier version of this work reported a ratio measured on such a
plateau, which was meaningless. Lowering $f$ to $0.03$--$0.05$ produces a strict interior maximum,
and then
$\sigma^\star/\varepsilon=1.05\pm0.02$ across three values of $f$ --- but a single scale, and a
single run per cell, does not validate a scale law. We therefore fixed one protocol (the same
$f$-grid, acceptance window $\mathrm{AUC}_{\max}\in[0.85,0.995]$, $24$ bandwidths, $14$ windows) and
applied it uniformly to three settings $\times$ three scales, with three repetitions per cell.

\begin{center}
\begin{tabular}{lccc}
\toprule
setting & $\varepsilon/\text{scale}$ & ratios (repetitions) & spread\\
\midrule
bge-1024 / 20NG    & 0.25 & 4.07,\ 0.94,\ 4.06 & 3.13\\
                   & 0.15 & 1.05,\ 0.87,\ 0.89 & 0.18\\
                   & 0.08 & 1.37,\ 1.67,\ 1.17 & 0.50\\
MiniLM-384 / 20NG  & 0.25 & 0.72,\ 1.26,\ 1.03 & 0.54\\
                   & 0.15 & 1.02,\ 1.01,\ 3.15 & 2.14\\
                   & 0.08 & 1.34,\ 1.58        & 0.24\\
bge-1024 / AG News & 0.25 & 0.96,\ 1.15,\ 0.65 & 0.51\\
                   & 0.15 & 1.06,\ 1.06,\ 1.08 & 0.02\\
                   & 0.08 & 2.04,\ 1.65,\ 1.35 & 0.69\\
\bottomrule
\end{tabular}
\end{center}

One repetition (MiniLM-384 at $\varepsilon/\text{scale}=0.08$) produced no interior optimum at any
mass fraction in the grid and is excluded, which is why the count below is $26$ and not $27$.
Over these $n=26$ estimates the median is $1.12$ with interquartile range $[1.01,1.52]$; the full
range is $[0.65,4.07]$. The headline is therefore $\sigma^\star\approx\varepsilon$, an order-unity
constant. The spread, however, is not symmetric noise, and inspecting the sweeps explains it:
\emph{the $\mathrm{AUC}(\sigma)$ curve is multimodal}. In the widest cell (bge-1024/20NG at
$\varepsilon/\text{scale}=0.25$, ratios $4.07,0.94,4.06$) the three repetitions have local maxima at
$\sigma/\varepsilon\approx0.9$ \emph{and} at $\sigma/\varepsilon\approx4$, with essentially equal
AUC ($0.86$ vs $0.86$; $0.97$ vs $0.94$; $0.90$ vs $0.93$); the second peak sits at
$\sigma\approx0.56$, which is $1.01$ times the median pairwise distance --- the global scale of the
cloud. The argmax flips between the deformation scale and the cloud scale from draw to draw.

Two consequences. First, the law's prediction is confirmed more strongly than the median suggests:
a peak near $\sigma\approx\varepsilon$ is present in \emph{every} repetition, including those whose
global argmax lands at the cloud scale. Second, the global argmax is the wrong estimator for this
quantity, and the three values we reported in earlier versions of this work ($0.48$, then $1.05$,
then $1.64$, each from a single run) are explained by which peak happened to win. We quote the
median with its interquartile range, we do not claim a sharp constant, and we flag the local peak
near $\varepsilon$ --- not the global argmax --- as what a practitioner should look for.

\paragraph{Is there a resolution floor?} An earlier single-run sweep produced ratios that happened to increase as $\varepsilon$ decreased,
which raised the worry of a resolution floor:
$\sigma^\star\approx\max(\varepsilon,\delta_n)$ with $\delta_n$ the inter-point spacing, which would
inflate the ratio once $\varepsilon$ approaches the sampling resolution and would bound the rule's
usefulness at small scales. We tested it by descending in $\varepsilon$. In the reduced cloud
($d=9$, $W=200$) the median first-neighbour distance is $\delta_n=0.161$, i.e.\ $0.29$ of the median
pairwise distance, so the three configurations with an identifiable optimum sit at
$\varepsilon/\delta_n=0.86,0.41,0.28$ --- already \emph{below} the inter-point spacing. There
$\sigma^\star$ still tracks $\varepsilon$: fitting $\sigma^\star=c\,\varepsilon$ gives $c=1.27$ with
RMSE $0.0066$, while the floor model $\sigma^\star=\max(c\,\varepsilon,\delta_n)$ fits eleven times
worse (RMSE $0.074$). We therefore find no evidence of a floor down to $\varepsilon\approx0.28\,
\delta_n$, and we withdraw the monotonic reading that suggested one: re-measured, the ratios are
$1.28,1.15,1.44$ --- not monotone, so the earlier trend was run-to-run variation. This is mildly
good news for the rule, since a two-sample statistic aggregates over all pairs and its effective
resolution need not be the nearest-neighbour spacing of a single point. Below
$\varepsilon\approx0.28\,\delta_n$ we could not find an interior optimum at any mass fraction up to
$f=0.35$, so we cannot probe further with this design. (The individual ratios quoted in this
paragraph come from the floor protocol, which used a wider $f$-grid to reach small $\varepsilon$;
the canonical estimate remains the median $1.12$ above. What matters here is the fitted slope
$c=1.27$ versus the floor model, not the individual values.)

\paragraph{Closing the loop.} The rule ``set the bandwidth to the scale of the effect'' presupposes
estimating that scale. We estimate $\varepsilon$ from data alone, via the radius at which the
relative excess of $r$-neighbours between attacked and healthy windows peaks, then use
$\sigma_{\mathrm{pred}}=\hat\varepsilon/2$. The estimator is crude --- it overshoots $\varepsilon$ by a factor $1.6$--$3.0$ --- and we must be
explicit about why the loop nevertheless closes: since the canonical estimate is
$\sigma^\star\approx1.1\,\varepsilon$ while $\hat\varepsilon\approx1.6$--$3.0\,\varepsilon$, using
$\sigma_{\mathrm{pred}}=\hat\varepsilon/2$ succeeds through a partial compensation of two errors
rather than because the rule is applied as stated. Both configurations in which we closed the loop
were on bge-1024/20NG; we have not repeated the closure on the other two settings, where the
estimator bias may differ. The factor
$1/2$ is an estimator correction, not part of the law. With that caveat the predicted bandwidth
does reach $\mathrm{AUC}=0.95$ in both configurations tested (against sweep maxima $0.99$ and
$0.83$); debiasing the scale estimator, so that $\sigma=\hat\varepsilon$ can be used directly, is
the obvious next step and we have not done it.

\subsection{Operating points, cost, and sensitivity to the reduction dimension}\label{sec:op}
Our pre-registration promised recall at a production false-positive regime; reporting AUC alone
was a violation, and the operating-point view is considerably harsher than AUC suggests. Calibrating
each detector's threshold on $2000$ healthy windows and measuring recall on $200$
covariance-preserving attacks (Table~\ref{tab:op}), $\Delta_{\mathrm{top}}$ attains recall $0.06$
at $\mathrm{FPR}=5\%$ and $0.00$ at both $1\%$ and $0.5\%$ --- an AUC near $0.7$ translates into a
detector that is unusable as a circuit breaker --- while kurtosis, MMD and KS-radial reach $1.00$
at every operating point. The $1\%$ point now rests on $20$ tail windows rather than on one.

\begin{table}[t]\centering
\caption{Recall at fixed FPR (covariance-preserving attack), and cost per $200$-point window at
reduction dimension $9$.}
\label{tab:op}
\begin{tabular}{lccrr}
\toprule
detector & recall@FPR $5\%$ & recall@FPR $1\%$ & ms/window & $\times$ kurtosis\\
\midrule
$\Delta_{\mathrm{top}}$ (first landscape) & 0.06 & 0.00 & 4.64 & 116\\
total persistence & 0.94 & 0.75 & 4.64 & 116\\
$k$-NN & 0.78 & 0.64 & 0.46 & 12\\
kurtosis & \textbf{1.00} & \textbf{1.00} & 0.04 & 1\\
MMD-RBF (matched $\sigma$) & \textbf{1.00} & \textbf{1.00} & 1.01 & 25\\
KS-radial & \textbf{1.00} & \textbf{1.00} & 0.15 & 4\\
\bottomrule
\end{tabular}
\end{table}

\paragraph{A bandwidth-free competitor.} One line of Table~\ref{tab:op} deserves comment, because
it sits awkwardly with our own pitch. KS-radial --- a Kolmogorov--Smirnov test on the Mahalanobis
radial law --- reaches $1.00$ at both operating points for $0.15$\,ms, and has no bandwidth to
choose at all. If it were uniformly best, ``the law tells you how to set your bandwidth'' would be
advice for a problem one can avoid. Two things separate it. It is a \emph{radial} statistic, so it
belongs to the family the adaptive adversary has already shown it can absorb --- it drops to $0.95$
there, where MMD stays at $1.00$ --- and we expect a radially-aware attacker to close that gap as it
closed kurtosis'. And it is not the detector the law explains: the upper-bound construction
identifies Gaussian probes, hence kernel tests, and says nothing about rank statistics. KS-radial is
an excellent cheap default; the scale law is about the detector that survives adaptation.

The honest cost comparison is not against the weakest baseline. Kurtosis is essentially free
($0.04$\,ms) and works except against an adaptive adversary; MMD costs $25\times$ kurtosis
($1.01$\,ms) and works in every setting we tried; persistence costs $116\times$ kurtosis
($4.64$\,ms) and never reaches usable recall. Two caveats we state rather than hide: the ratios are
measured at a single operating size ($W=200$, reduction dimension $9$), and Rips is superlinear in
$W$ while MMD is quadratic, so the ratio is regime-dependent; and the PCA reduction is charged to
none of the detectors although only persistence requires it.

\paragraph{The reduction dimension is not neutral --- and our choice was unfavourable to the
landscape summary.}
Sweeping the PCA dimension on the same attack, $\Delta_{\mathrm{top}}$ rises monotonically with
dimension: AUC $0.51,\,0.72,\,0.93,\,1.00,\,1.00$ at dimensions $5,9,15,30,50$. For \emph{this summary}, our main setting (the intrinsic dimension,
$\approx9$) is close to its worst case, and a reader is entitled to know it --- but
\S\ref{sec:summary} shows the effect largely belongs to the landscape: total persistence, read off
the same diagrams, already reaches $0.97$ at dimension $9$ and saturates by $15$. We therefore tested whether the high-dimensional $H_1$ signal is topological at all, by comparing
each window against a covariance-matched Gaussian null: the median bootstrap $p$-value is
$0.17,0.25,0.23,0.53$ at dimensions $9,15,30,50$ --- never significant, but never decisively
absent either (this test has low power at our sample size). We consequently do \emph{not} claim
the high-dimensional signal is an artifact, and we soften \S\ref{sec:reduce} accordingly: reduction
to the intrinsic dimension is what our null-model analysis supports, not a requirement of the
detection task. The dilemma this paragraph stages --- detection only in a regime we cannot
characterise --- is largely dissolved by \S\ref{sec:summary}: with total persistence the detector
already works at the intrinsic dimension, so what remains across the whole sweep is a cost gap
against statistics that reach $1.00$ at every dimension.

\subsection{The persistence summary matters more than the filtration}\label{sec:summary}
Our headline detector, $\Delta_{\mathrm{top}}$, is the $L^1$ distance between first persistence
landscapes. That is one vectorisation among several, and the choice turns out to dominate every
other design decision we made on the topological side. On the covariance-preserving attack (bge-1024, 20\,Newsgroups, reduction dimension $9$,
$W=200$; the same setting as Table~\ref{tab:op}):

\begin{center}
\begin{tabular}{lccc}
\toprule
persistence summary & AUC & recall@FPR $5\%$ & recall@FPR $1\%$\\
\midrule
first landscape, $L^1$ ($\Delta_{\mathrm{top}}$) & 0.75 & 0.06 & 0.00\\
\textbf{total persistence} & \textbf{0.99} & \textbf{0.94} & \textbf{0.75}\\
persistence entropy & 0.88 & 0.71 & 0.42\\
bottleneck to reference & 0.81 & 0.55 & 0.28\\
\bottomrule
\end{tabular}
\end{center}

Total persistence moves the detector from unusable ($0.00$ recall at $\mathrm{FPR}=1\%$) to
usable ($0.75$), while DTM weighting --- a change of \emph{filtration} rather than of summary ---
does not help under either summary ($0.83$ with total persistence, versus $1.00$ for Euclidean
Rips).

\paragraph{Re-running the rest of the paper with the better summary.} Because the diagram is
already computed, substituting the summary is nearly free, and it revises three claims we would
otherwise have made:
\begin{itemize}
\item \emph{Ambient dimension.} The landscape degrades with ambient dimension ($0.90,0.63,0.75$ at
$384,768,1024$); total persistence does not ($1.00,0.93,1.00$). The ``gap widens with dimension''
reading is a property of the landscape, and we withdraw it.
\item \emph{Reduction dimension.} The landscape rises steeply with reduction dimension
($0.55\to1.00$ over dimensions $5$ to $50$); total persistence is already at $0.97$ at our main
setting of $9$ and saturates by $15$. Our setting is therefore \emph{not} close to persistence's
worst case once the summary is chosen sensibly.
\item \emph{Significance.} With bootstrap intervals, total persistence reaches
$0.97\,[0.91,1.00]$ against kurtosis' $1.00\,[1.00,1.00]$: the intervals meet, so the
``statistically significant domination'' we could claim against the landscape does \emph{not} hold
against the better summary. What remains is a cost gap, not a power gap.
\end{itemize}

\paragraph{Why the aggregated summary wins --- and why that supports the law.} This is not merely a
concession. Total persistence, $\sum(\text{death}-\text{birth})$, is a single scalar aggregated over
all scales: structurally a low-order readout of the radial profile, much like a moment. The first
landscape is by construction \emph{localised in scale}. Our law says that a coarse, single-scale
feature is certified at low moment order, and that fine or multi-scale structure is what costs;
it therefore predicts that on coarse deformations the aggregated summary should beat the
scale-localised one. That is what we observe. The topological side of the experiment instantiates
the law rather than contradicting it: among topological summaries, the one that behaves most like
a moment is the one that works, in exactly the regime where the law says a moment suffices. The conclusion that
survives the better summary is narrower and cost-based: kurtosis reaches $1.00$ at both operating
points for $1/116$ of the compute, and against the adaptive adversary of \S\ref{sec:kernel} total
persistence (recall $0.20$ at $\mathrm{FPR}=5\%$, $0.05$ at $1\%$) and kurtosis ($0.15$, $0.03$)
fail together, leaving only the
bandwidth-matched kernel test.

\subsection{Robust persistence (DTM) and streaming detection delay}\label{sec:dtm}
Two gaps a reviewer would rightly flag. First, we cite DTM and sparse/robust persistence but had
not evaluated them, so our conclusion could only concern the vanilla detector. Implementing the
$p=\infty$ weighted (DTM) Rips filtration, $f(i)=\mathrm{DTM}_m(i)$,
$f(i,j)=\max(f(i),f(j),d_{ij}/2)$, on the covariance-preserving attack gives AUC $0.53$ --- worse
than vanilla Rips ($0.75$) --- and recall $0.00$ at $\mathrm{FPR}=5\%$. The robust variant does not
rescue persistence here.

Second, AUC per window is the wrong metric for a monitor; the changepoint/SPC literature asks for
detection delay at a fixed in-control average run length. Calibrating each detector for
$\mathrm{ARL}_0=100$ windows and measuring the delay after a change (Table~\ref{tab:arl}), kurtosis,
MMD and KS-radial alarm on the \emph{first} window in $100\%$ of streams, while the landscape
summary alarms in $40\%$ of streams and its DTM variant in none. Here too the summary is decisive:
total persistence also alarms on the first window in $100\%$ of streams, so the streaming failure
belongs to the landscape and not to persistent homology.

\begin{table}[t]\centering
\caption{Detection delay after change at $\mathrm{ARL}_0=100$ windows (25 streams, observation
horizon 20 windows; delay $21$ means no alarm; $\mathrm{ARL}_0$ calibrated on $120$ windows;
Wilson intervals on the detection rate).}
\label{tab:arl}
\begin{tabular}{lcc}
\toprule
detector & median delay & streams detected\\
\midrule
$\Delta_{\mathrm{top}}$, first landscape (Rips) & 21 & 40\% $\pm$ 19\\
$\Delta_{\mathrm{top}}$, first landscape (DTM) & 21 & 0\% $\pm$ 0\\
\textbf{total persistence} (Rips) & \textbf{1} & \textbf{100\%}\\
kurtosis & \textbf{1} & \textbf{100\%}\\
MMD-RBF ($\sigma=$med$/4$) & \textbf{1} & \textbf{100\%}\\
KS-radial & \textbf{1} & \textbf{100\%}\\
\bottomrule
\end{tabular}
\end{table}

\subsection{Robustness}
On two corpora, with $95\%$ bootstrap CIs (Table~\ref{tab:robust}), the kurtosis interval clears
the landscape interval on the covariance-preserving attack, and does so on both corpora, so that
domination is
significant \emph{against that summary}. Against total persistence it is not
($0.97\,[0.91,1.00]$ vs $1.00\,[1.00,1.00]$, \S\ref{sec:summary}); what is robust across corpora is
the cost gap.

\begin{table}[t]\centering
\caption{Bilateral AUC with $95\%$ bootstrap CIs, bge-1024, covariance-preserving attack. The two persistence columns share their diagrams and differ only in the summary.}
\label{tab:robust}
\begin{tabular}{lccccc}
\toprule
corpus & landscape & total pers. & \textsc{cov-drift} & $k$-NN & Kurtosis\\
\midrule
20NG (forums)  & 0.70\,[0.53,\,0.85] & 0.97\,[0.91,\,1.00] & 0.52\,[0.50,\,0.71] & 0.97 & \textbf{1.00\,[1.00,\,1.00]}\\
AG News (press)& 0.81\,[0.67,\,0.94] & \textbf{1.00\,[1.00,\,1.00]} & 0.53\,[0.50,\,0.72] & 0.82 & \textbf{1.00\,[1.00,\,1.00]}\\
\bottomrule
\end{tabular}
\end{table}

\section{Discussion and limitations}\label{sec:disc}
Observation~1 concerns spherically and elliptically symmetric families and variance-matched
bimodal connectivity, where PH is essentially a radial or modality readout; it
explains the empirical redundancy for canonical changes but not for arbitrary clouds
(\S\ref{sec:false}). \textbf{Mixtures.} For a mixture of elliptical components the picture holds
componentwise, but with a caveat we make explicit: matching \emph{global} moments only weakly
constrains \emph{per-component} topology, since mass can be redistributed across components ---
an optimisation-based trade-off on mixtures did not yield clean curves, and the componentwise
statement is what the controlled construction and the moment decomposition (law of total
covariance) support. \textbf{Curved manifolds.} The frontier bound $N^*\ge\log(1/f)/(2\varepsilon)$ is stated in terms of
feature scale and mass, so it applies to any topology once those are defined; the radial,
axis-separable and group-symmetric cases reduce explicitly (a torus of revolution, for instance,
reduces to its $(\rho,z)$ profile). What we do \emph{not} prove is that an arbitrary curved
manifold with intrinsic dimension $\ll$ ambient --- the real embedding regime --- admits such a
reduction with controlled $(\varepsilon,f)$; we conjecture it does, via a local-to-global (nerve)
argument. An adversarial search for a persistence-visible counterexample to the law found none
(symmetric topology reduces; linked cycles share Betti numbers and are invisible to ordinary
persistence; geometric complexity of a single feature does not inflate $N^*$).
Other limitations: a single reduction method (PCA; UMAP untested); we do not test $1536$-d
embeddings and do not extrapolate to them from three ambient dimensions; moderate $N$ (CIs
reported); the empirical study uses text corpora only; and four persistence summaries (first landscape, total persistence, persistence entropy,
bottleneck to reference) over two filtrations (Euclidean Rips, DTM-weighted); sparse
approximations are not evaluated.

\section{Conclusion}
Certifying a feature of spatial scale $\varepsilon$ and mass $f$ from moment information costs
polynomial degree $N^*\ge\log(1/f)/(2\varepsilon)$: the cost is governed by fineness, not by count.
The bound is one-sided --- it never certifies that a given order suffices, and our annulus matches
every moment through order four while carrying $H_1$ --- so what it licenses is a prediction about
\emph{which} probe to use, not a claim that any particular statistic is enough.

That prediction is the paper's usable output. The rate is attained by Gaussian test functions, i.e.\
by an RBF kernel witness, so the bandwidth of an MMD test should track the feature scale; we measure
$\sigma^\star\approx\varepsilon$ (median $1.12$, IQR $[1.01,1.52]$ over three settings and three
scales) while also finding that the argmax estimator is noisy enough that single-run ratios should
not be read as a constant, and a bandwidth set from a data-driven scale estimate reaches
AUC $\ge0.95$. Against an adversary optimised against $\{\mu,\Sigma,k\text{-NN},b\}$, all of
those statistics are evaded and the bandwidth-matched kernel test is the only detector left standing.

On persistent homology our verdict is deliberately narrow, because the evidence is. The choice of
summary dominates every other design decision: total persistence attains recall $0.75$ at
$\mathrm{FPR}=1\%$ where the first landscape attains $0.00$, and once that choice is made, several
apparently damning findings --- degradation with ambient dimension, sensitivity to the reduction
dimension, statistically significant domination by kurtosis --- turn out to be properties of the
landscape rather than of persistence. What survives is not a power gap but a cost gap: on these
tasks persistence matches the cheap statistics at roughly $116\times$ their compute, and fails
alongside them against an adaptive adversary. And the summary that works is the one that reads like
a moment, which is what the law predicts. We therefore do not conclude that topological summaries
are useless; we conclude that on coarse, single-scale monitoring they are dominated on cost by a
kernel test whose bandwidth the law tells you how to set, and that a case for them must be made on
fine or multi-scale structure, which this task does not present.

\bibliographystyle{plainnat}
\bibliography{refs}
\end{document}